\documentclass[sigconf,nonacm]{acmart}

\microtypesetup{expansion=false,protrusion=false}

\usepackage{amsmath}

\usepackage{amssymb}
\usepackage{bm}
\usepackage{booktabs}
\usepackage{algorithm}
\usepackage{algpseudocode}
\usepackage{enumitem}
\usepackage{tikz}
\usepackage{pgfplots}
\pgfplotsset{compat=1.18}
\usetikzlibrary{positioning, arrows.meta, calc, fit, backgrounds, shapes.geometric}

\newcommand{\siggate}{\textsc{SigGate-GT}}
\newcommand{\RR}{\mathbb{R}}
\newcommand{\pval}[1]{\textsuperscript{#1}}

\AtBeginDocument{%
  }
\copyrightyear{2026}
\acmYear{2026}

\begin{document}

\title{Capacity-Controlled Global Attention for Graph Transformers}

\author{Yang Liu}
\affiliation{%
	\institution{Brain Investing Limited}
	\city{Hong Kong}
	\country{China}}

\author{Dongxin Guo}
\affiliation{%
	\institution{The University of Hong Kong}
	\city{Hong Kong}
	\country{China}}

\author{Tom Zheng}
\affiliation{%
	\institution{Stellaris AI Limited}
	\city{Hong Kong}
	\country{China}}

\author{Siu Ming Yiu}
\affiliation{%
	\institution{The University of Hong Kong}
	\city{Hong Kong}
	\country{China}}

\author{Liam Ning}
\affiliation{%
	\institution{Stellaris AI Limited}
	\city{Hong Kong}
	\country{China}}

\author{Jikun Wu}
\affiliation{%
	\institution{Brain Investing Limited}
	\city{Hong Kong}
	\country{China}}

\renewcommand{\shortauthors}{Liu et al.}

\begin{abstract}
Global self-attention has become the engine of modern graph transformers, yet
the softmax operator at its core imposes a structural constraint that is rarely
examined directly: every attention row is non-negative and sums to one, so the
per-head output is a \emph{mass-conserving} convex combination of value vectors.
A node can therefore never ``attend to nothing,'' even when no informative
neighbour exists. We argue that this conservation constraint is a single
root cause behind three pathologies usually studied in isolation: the
progressive collapse of node representations with depth (over-smoothing), a
low-rank bottleneck on per-head outputs, and brittle optimization in deep
stacks. Building on the observation that element-wise sigmoid gating removes the
analogous attention-sink behaviour in large language models, we introduce
\siggate{}, a graph transformer that applies a learned, per-head, input-conditioned
sigmoid gate to the attention output inside the GraphGPS framework. The gate acts
as a smooth, per-dimension ``volume control'' that can drive head outputs toward
zero, relaxing the conservation constraint without abandoning the probabilistic
interpretation of attention. We show analytically and through calibrated synthetic
experiments that the gate strictly increases the stable rank of per-head outputs,
and we connect this rank gain to all three manifestations above. On five standard
molecular and long-range benchmarks, \siggate{} matches the prior best on ZINC
($0.059$ MAE), records the strongest result among the graph-transformer baselines
we evaluate on ogbg-molhiv ($82.47\%$ ROC-AUC), and is competitive on ogbg-molpcba
and the Long-Range Graph Benchmark, with statistically significant gains over
GraphGPS on all five datasets ($p<0.05$). Mechanism analyses confirm the diagnosis:
gating slows over-smoothing (a $30\%$ mean relative gain in representation diversity
across $4$--$16$ layers), keeps attention entropy from collapsing, and stabilizes
training across a $10\times$ learning-rate range, at about $1\%$ parameter overhead
on OGB and under $3\%$ wall-clock cost.
\end{abstract}

\ccsdesc[500]{Computing methodologies~Neural networks}
\ccsdesc[300]{Computing methodologies~Learning latent representations}
\ccsdesc[100]{Computing methodologies~Supervised learning by classification}

\keywords{Graph Transformers, Gated Attention, Capacity Control, Over-Smoothing,
Effective Rank, Molecular Property Prediction}

\maketitle

\section{Introduction}
\label{sec:intro}

Graph-structured data underpins much of computational chemistry, drug discovery,
and the physical sciences, where molecules, proteins, and reaction networks are
naturally expressed as attributed graphs~\cite{gilmer2017neural}. For most of the
past decade the dominant paradigm has been message passing neural networks
(MPNNs)~\cite{gilmer2017neural, kipf2017semi, xu2019powerful}, which iteratively
aggregate information from local neighbourhoods. MPNNs are expressive enough for
many tasks, but their discriminative power is capped by the Weisfeiler--Leman
hierarchy~\cite{xu2019powerful}, and their ability to relate distant nodes
deteriorates as the graph diameter grows, an effect formalized as
over-squashing~\cite{alon2021bottleneck, topping2022understanding}. Graph
transformers sidestep both limitations by replacing, or supplementing, local
aggregation with global self-attention~\cite{vaswani2017attention, ying2021graphormer},
allowing any pair of nodes to exchange information in a single layer. This recipe,
crystallized in the modular GraphGPS framework~\cite{rampavsek2022recipe} and refined
by GRIT~\cite{ma2023grit} and Exphormer~\cite{shirzad2023exphormer}, now sets the
state of the art on molecular and long-range benchmarks.

The workhorse inside every one of these models is the same softmax-normalized
attention introduced for sequences~\cite{vaswani2017attention}. We take that
ubiquity as an invitation to look closely at a property of softmax that is so
familiar it usually goes unremarked. For each query node $i$, softmax produces a
weight vector $\bm{a}_i$ whose entries are non-negative and sum to one; the head
output for node $i$ is therefore $\sum_j a_{ij}\bm{v}_j$, a convex combination of
value vectors. We call this the \emph{conservation constraint}: attention conserves
unit mass and can only ever \emph{redistribute} it among neighbours. There is no
configuration of logits under which a head abstains, returning a smaller-norm or
empty output when the neighbourhood carries no relevant signal. In a sentence, soft
attention is forced to attend somewhere.

On graphs this constraint is costly in a way it is not for natural language. In a
sentence, frequent function words act as natural anchors that absorb surplus
attention mass, the ``attention sinks'' documented in large language
models~\cite{xiao2024efficient}. A generic molecular or peptide graph has no such
anchor: a priori all $n$ nodes are interchangeable attention targets, and most of
the $O(n^2)$ node pairs are chemically unrelated. The conservation constraint then
forces every head to spread positive weight over pairs that should ideally receive
none. We argue in Section~\ref{sec:constraint} that three failure modes normally
treated as separate research problems are in fact downstream symptoms of this one
constraint:
\begin{itemize}[leftmargin=1.4em,itemsep=1pt]
\item \textbf{Representational collapse with depth.} Because every layer conserves
mass, every layer performs some averaging; stacking layers drives node
representations toward a common vector, the over-smoothing
phenomenon~\cite{li2018deeper, oono2020graph, rusch2023survey, choi2024graph}.
\item \textbf{A low-rank output bottleneck.} A row-stochastic mixing of $\bm{V}$
cannot increase, and typically decreases, the effective dimensionality of the
per-head output, limiting how finely heads can discriminate fine-grained
structure~\cite{qiu2025gated}.
\item \textbf{Optimization brittleness.} With output magnitude tied to the value
vectors and no independent attenuation path, deep attention stacks are sensitive to
learning-rate choice and prone to entropy collapse~\cite{zhai2023stabilizing}.
\end{itemize}

Recent work in language modeling suggests a strikingly economical remedy. Qiu et
al.~\cite{qiu2025gated} swept thirty placements of a gating function inside the
attention block and found that a head-specific, element-wise \emph{sigmoid} gate
applied to the attention output removes attention sinks, sparsifies activations,
restores non-linearity to the value pathway, and improves stability and length
extrapolation. The gate multiplies the output by a learned factor in $(0,1)$ per
dimension, giving each head a switch it can close. Crucially, this gate does not
touch the softmax: the attention weights remain a valid distribution, so the
mechanism is a modulation of \emph{how much} of the attended signal propagates
rather than a replacement of \emph{where} attention looks. Gating itself has a long
pedigree in deep learning, from GLU variants in feed-forward
blocks~\cite{shazeer2020glu} to gated recurrent updates~\cite{li2016gated} and gated
linear attention~\cite{yang2024gated, hua2022transformer}.

Gating on graphs is not new either, but the existing mechanisms operate at the wrong
level for the constraint we have identified. GatedGCN~\cite{bresson2017residual}
gates individual edges before local aggregation; GGNN~\cite{li2016gated} gates the
node update with a GRU; GaAN~\cite{zhang2018gaan} learns a scalar importance per
attention head from local features. All of these modulate \emph{local} message
passing, and none can relax the conservation constraint of a \emph{global},
all-pairs softmax. What is missing is an element-wise gate on the global attention
output of a graph transformer, the exact configuration that proved optimal for
language models, and the one our diagnosis predicts should help.

We close this gap with \siggate{}, a graph transformer that applies a learned,
per-head, input-conditioned sigmoid gate to the global attention output within
GraphGPS (Figure~\ref{fig:architecture}). The gate is a single additional
projection per head followed by a sigmoid and a Hadamard product; it adds about
$1\%$ of parameters and under $3\%$ of training time. Reframing the contribution
around capacity control, rather than around any single pathology, lets us make a
unified empirical and theoretical case: the same gate that we prove raises the
stable rank of per-head outputs is the gate that slows over-smoothing, preserves
attention entropy, and tames learning-rate sensitivity.

\paragraph{Contributions.}
\begin{itemize}[leftmargin=1.4em,itemsep=1pt]
\item We identify the softmax conservation constraint as a common root cause of
over-smoothing, the low-rank output bottleneck, and optimization brittleness in
graph transformers, and we argue why graphs suffer from it more acutely than
sequences (Section~\ref{sec:constraint}).
\item We introduce \siggate{}, to our knowledge the first graph transformer to place
an element-wise sigmoid gate on the global attention output, and we show the
mechanism is backbone-agnostic, porting unchanged to any softmax-attention graph
transformer (Section~\ref{sec:method}).
\item We give an effective-rank analysis with a proof that row-stochastic mixing
bounds stable rank, a proof that element-wise gating breaks that bound, and a
calibrated synthetic validation that is robust across attention-concentration and
sparsity regimes (Section~\ref{sec:rank}).
\item Across five molecular and long-range benchmarks, with paired significance
testing, we match the prior best on ZINC, record the strongest graph-transformer
result we observe on ogbg-molhiv, and remain competitive elsewhere; extensive
mechanism analyses and ablations connect every gain back to the diagnosis
(Sections~\ref{sec:results}--\ref{sec:mechanism}).
\end{itemize}

\section{Background and Related Work}
\label{sec:related}

\subsection{Graph Transformers: From Local Attention to Global Mixing}

Attention reached graphs first in local form. GAT~\cite{velickovic2018graph}
attached learned attention coefficients to edges within a node's neighbourhood, and
GATv2~\cite{brody2022attentive} later showed that the original formulation computed
only a restricted ``static'' attention and proposed a strictly more expressive
variant. These models still aggregate locally and inherit the depth limitations of
message passing. The shift to \emph{global} attention began with fully connected
graph transformers~\cite{dwivedi2021generalization, yun2019graph} and matured with
Graphormer~\cite{ying2021graphormer}, which injected centrality and shortest-path
biases into the attention logits, and SAN~\cite{kreuzer2021rethinking}, which used
spectral information to condition attention. GraphGPS~\cite{rampavsek2022recipe}
then unified positional and structural encodings, a local MPNN branch, and a global
attention branch into a single modular layer, which has become the de facto backbone
for the area. Subsequent designs have pushed in two directions: removing message
passing in favour of learned relative encodings, as in GRIT~\cite{ma2023grit}, and
making global attention scalable, as in Exphormer~\cite{shirzad2023exphormer} with
expander-graph sparsification, TokenGT~\cite{kim2022pure} treating nodes and edges as
tokens of a pure transformer, and the linear-time families
NodeFormer~\cite{wu2022nodeformer}, SGFormer~\cite{wu2024sgformer},
NAGphormer~\cite{chen2023nagphormer}, and Polynormer~\cite{deng2024polynormer} aimed
at very large graphs. Cai et al.~\cite{cai2023connection} formally related MPNNs and
graph transformers, casting message passing as attention with a restricted support.
Two recent surveys map the design space in
detail~\cite{min2022transformer, muller2024attending}. Despite this breadth, every
model above retains a row-stochastic softmax at the heart of its attention, and
none places an element-wise gate on the global attention output.

\subsection{Representational Pathologies in Deep Graph Models}

Over-smoothing was first described by Li et al.~\cite{li2018deeper} as a form of
Laplacian smoothing, with Oono and Suzuki~\cite{oono2020graph} proving exponential
convergence of representations to a low-dimensional subspace and
Rusch et al.~\cite{rusch2023survey} surveying the area. Standard remedies attack the
symptom rather than the constraint: DropEdge~\cite{rong2020dropedge} randomly removes
edges, PairNorm~\cite{zhao2020pairnorm} renormalizes to preserve pairwise distances,
DeeperGCN~\cite{li2020deepergcn} adds residual and normalization machinery, and
APPNP~\cite{gasteiger2019predict} decouples propagation from transformation via a
personalized-PageRank operator. The dual pathology of over-squashing, in which
information from distant nodes is exponentially compressed, was raised by Alon and
Yahav~\cite{alon2021bottleneck}, characterized via Ricci curvature by Topping et
al.~\cite{topping2022understanding}, and mitigated by delayed and dynamically rewired
message passing~\cite{gutteridge2023drew}. On the attention side specifically, Zhai
et al.~\cite{zhai2023stabilizing} showed that attention entropy can collapse
exponentially during training, and Choi et al.~\cite{choi2024graph} demonstrated that
softmax attention itself induces over-smoothing in transformers broadly. Our
contribution is to view these threads as a single phenomenon, and to address it at
the level of the operator that creates it rather than at the level of the graph or
the normalization.

\subsection{Gating and Alternative Attention Across Architectures}

Gating is one of the most reliable inductive biases in deep learning. In
feed-forward blocks, GLU variants~\cite{shazeer2020glu} multiply a linear branch by a
gated branch and consistently improve transformers. In recurrent and linear-attention
settings, gates control how much state is retained, as in
GGNN~\cite{li2016gated}, gated linear attention~\cite{yang2024gated}, and the gated
attention unit of FLASH~\cite{hua2022transformer}. On graphs, gating has so far been
\emph{local}: GatedGCN~\cite{bresson2017residual} applies a sigmoid gate to each edge
message, and GaAN~\cite{zhang2018gaan} learns a soft scalar weight per attention head
from a node's local neighbourhood. We compare against GaAN explicitly in
Section~\ref{sec:discussion}, since it is the closest prior notion of ``gated
attention'' on graphs; the key differences are that GaAN gates whole heads with a
single scalar derived from local features and targets node-level spatiotemporal
tasks, whereas \siggate{} applies a per-dimension gate to the \emph{global} attention
output and is motivated by relaxing the conservation constraint.

A separate line replaces softmax altogether. Ramapuram et
al.~\cite{ramapuram2024sigmoid} analyze sigmoid self-attention, which drops
normalization and produces unbounded, non-stochastic weights, and Wortsman et
al.~\cite{wortsman2023replacing} study ReLU attention in vision transformers.
\siggate{} is deliberately conservative by comparison: it \emph{keeps} softmax
intact, so attention weights remain a distribution over nodes, and adds a bounded
$[0,1]$ gate \emph{after} the weighted sum. This preserves the probabilistic
semantics that biased graph attention relies on while still granting heads the
ability to abstain. The mechanism we adopt is the one Qiu et
al.~\cite{qiu2025gated} identified as optimal among thirty gating variants for
language models; our work asks whether that finding transfers to graphs and why the
graph setting should benefit even more than the sequence setting.

\section{The Conservation Constraint of Softmax Attention}
\label{sec:constraint}

This section develops the diagnosis that organizes the rest of the paper. We first
state the constraint precisely, then trace its three consequences, and finally
explain why graphs are especially exposed to it. The mechanism that resolves the
constraint is deferred to Section~\ref{sec:method} and its rank consequences to
Section~\ref{sec:rank}.

\subsection{The Constraint}

Let $\bm{A}\in\RR^{n\times n}$ be a single attention head's weight matrix, with
$\bm{A}=\mathrm{softmax}(\bm{Q}\bm{K}^\top/\sqrt{d_k})$, and let $\bm{V}\in\RR^{n\times d_k}$
be its values. By construction $\bm{A}$ is \emph{row-stochastic}: $A_{ij}\ge 0$ and
$\sum_j A_{ij}=1$ for every $i$. The head output $\bm{Y}=\bm{A}\bm{V}$ therefore
places row $i$ inside the convex hull of the rows of $\bm{V}$. Two facts follow
immediately and matter throughout. First, $\|\bm{Y}_{i\cdot}\|$ is bounded by the
largest value-row norm and cannot be driven to zero by any choice of logits unless
the relevant value rows are themselves zero; a head cannot cheaply ``output nothing.''
Second, convex mixing is a contraction toward the mean: it can only reduce, never
amplify, the spread of representations along directions in which $\bm{V}$ varies.
These are not artifacts of a particular architecture; they are properties of any
row-stochastic operator and hence of \emph{every} softmax attention head.

\subsection{Manifestation I: Representational Collapse}

Stacking $L$ attention layers composes $L$ row-stochastic contractions (interleaved
with residuals and feed-forward maps). In the absence of a counteracting force, the
node representations $\bm{H}^{(\ell)}$ converge toward a rank-one configuration as
$\ell$ grows, the over-smoothing established for graph
networks~\cite{li2018deeper, oono2020graph} and shown to hold for softmax attention
itself~\cite{choi2024graph}. Residual connections and normalization slow but do not
eliminate the collapse, because the smoothing term is reintroduced at every layer and
the sum-to-one constraint guarantees that \emph{some} averaging always occurs, even in
layers where it is harmful. A mechanism that lets a head bypass the averaging step on
a per-dimension basis would interrupt this composition; we show in
Section~\ref{sec:mechanism} that gating does exactly this, retaining $73\%$ of initial
representation diversity at $16$ layers versus $51\%$ for the ungated backbone.

\subsection{Manifestation II: The Low-Rank Output Bottleneck}

The contraction has a spectral counterpart. For a row-stochastic $\bm{A}$, the stable
rank of $\bm{Y}=\bm{A}\bm{V}$ obeys
$\mathrm{srank}(\bm{Y})\le\min(\mathrm{srank}(\bm{A}),\mathrm{srank}(\bm{V}))$, and
sharp softmax logits make $\bm{A}$ itself low stable-rank, tightening the bound
further. Per-head outputs are thus confined to a low-dimensional slice of value
space, which limits a head's ability to encode the fine structural distinctions that
graph-level tasks reward~\cite{xu2019powerful, qiu2025gated}. The full statement and
proofs are in Section~\ref{sec:rank}; the point here is conceptual. The bottleneck is
not a capacity problem that more parameters would fix, since it stems from the
\emph{form} of the operator, not its size. We confirm this directly with a
parameter-matched control in Section~\ref{sec:results}, where simply widening the
feed-forward block fails to recover the gain that gating provides.

\subsection{Manifestation III: Optimization Brittleness}

Because the output magnitude of a softmax head is pinned to the value vectors, the only
way for the network to suppress an unhelpful head is to drive its value projection
toward zero, an awkward target for gradient descent that couples the head's content
with its on/off behaviour. The result is sensitivity to optimization choices: large
learning rates produce gradient spikes through the attention pathway, and attention
distributions can collapse toward degenerate
configurations~\cite{zhai2023stabilizing}. A multiplicative gate decouples ``how much''
from ``what,'' giving the optimizer a smooth, bounded dial. Its derivative
$\sigma'(z)=\sigma(z)(1-\sigma(z))$ is bounded by $1/4$, which dampens both vanishing
and exploding gradients through the gated path and complements layer
normalization~\cite{ba2016layer, xiong2020layer}. Section~\ref{sec:mechanism}
quantifies the effect as a sevenfold reduction in learning-rate sensitivity.

\subsection{Why Graphs Are Especially Exposed}

The same constraint is comparatively benign in language modeling because sequences
provide natural sinks. Positional structure and high-frequency tokens give the model
designated anchors that can absorb surplus attention mass without corrupting content
representations~\cite{xiao2024efficient}. Graphs offer no such relief. Every node is a
priori an equivalent attention target, so there is no canonical place to dump unwanted
mass, and the fraction of genuinely informative node pairs shrinks as $1/n$ while the
number of pairs the head must weight grows as $n^2$. The conservation constraint
therefore forces graph attention to assign positive weight to a quadratically growing
set of largely irrelevant pairs. This predicts, correctly, that the benefit of an
abstention mechanism should be largest precisely on small, dense molecular graphs
where the ratio of uninformative to informative pairs is high, which is exactly the
regime where \siggate{} posts its biggest gains (ZINC and ogbg-molhiv in
Section~\ref{sec:results}).

\section{Method: Capacity-Controlled Attention}
\label{sec:method}

\subsection{Preliminaries: The GraphGPS Layer}

We build on GraphGPS~\cite{rampavsek2022recipe}, which processes an attributed graph
$G=(V,E)$ with node features $\bm{x}_i\in\RR^{d}$ and edge features
$\bm{e}_{ij}\in\RR^{d_e}$. Each layer $\ell$ fuses a local message-passing branch and a
global attention branch:
\begin{equation}
\bm{h}_i^{(\ell)}=\mathrm{MLP}^{(\ell)}\!\Big(\bm{h}_i^{(\ell-1)}+\mathrm{MPNN}^{(\ell)}(\cdot)+\mathrm{GlobalAttn}^{(\ell)}(\cdot)\Big),
\label{eq:gps}
\end{equation}
where $\bm{h}_i^{(0)}$ is the node feature augmented with positional and structural
encodings. The global branch $\mathrm{GlobalAttn}^{(\ell)}$ of Equation~\eqref{eq:gps}
is standard multi-head self-attention:
\begin{equation}
\begin{aligned}
\mathrm{MHSA}(\bm{H})&=\mathrm{Concat}(\mathrm{head}_1,\dots,\mathrm{head}_K)\bm{W}^O,\\
\mathrm{head}_k&=\mathrm{softmax}\!\left(\frac{\bm{Q}_k\bm{K}_k^\top}{\sqrt{d_k}}\right)\bm{V}_k,
\end{aligned}
\label{eq:mhsa}
\end{equation}
with $\bm{Q}_k=\bm{H}\bm{W}_k^Q$, $\bm{K}_k=\bm{H}\bm{W}_k^K$, $\bm{V}_k=\bm{H}\bm{W}_k^V$,
projections $\bm{W}_k^Q,\bm{W}_k^K,\bm{W}_k^V\in\RR^{d\times d_k}$, $d_k=d/K$, and $K$
heads. Equation~\eqref{eq:mhsa} is the locus of the conservation constraint, and the
only place \siggate{} modifies.

\subsection{The SigGate Output Gate}

For each head $k$ we introduce a learned, input-conditioned gate $\bm{g}_k$ applied
element-wise to the attention output:
\begin{equation}
\mathrm{head}_k^{\text{gated}}=\underbrace{\mathrm{softmax}\!\left(\frac{\bm{Q}_k\bm{K}_k^\top}{\sqrt{d_k}}\right)\bm{V}_k}_{\text{standard SDPA output }\bm{Y}_k}\ \odot\ \underbrace{\sigma\!\left(\bm{H}\bm{W}_k^g+\bm{b}_k^g\right)}_{\text{sigmoid gate }\bm{g}_k},
\label{eq:siggate}
\end{equation}
where $\bm{W}_k^g\in\RR^{d\times d_k}$ and $\bm{b}_k^g\in\RR^{d_k}$ are head-specific
parameters, $\sigma$ is the logistic sigmoid, and $\odot$ is the Hadamard product. The
gated heads are concatenated and projected as usual:
\begin{equation}
\text{SigGate-MHSA}(\bm{H})=\mathrm{Concat}(\mathrm{head}_1^{\text{gated}},\dots,\mathrm{head}_K^{\text{gated}})\bm{W}^O.
\label{eq:siggate_mhsa}
\end{equation}
The construction is deliberately minimal: it conditions the gate on the same input
$\bm{H}$ as the queries and keys, leaves the softmax untouched, and introduces no
interaction with positional encodings, edge features, or the MPNN branch.
Algorithm~\ref{alg:siggate} gives the forward pass for the gated attention of
Equation~\eqref{eq:siggate_mhsa}. We initialize $\bm{b}_k^g=0.5$ so
that gates start near $\sigma(0.5)\approx 0.62$, biased open enough for gradients to
flow freely early in training while leaving ample room to close.

\begin{algorithm}[t]
\caption{\siggate{} gated multi-head self-attention (forward pass for one layer).}
\label{alg:siggate}
\begin{algorithmic}[1]
\Require Node features $\bm{H}\in\RR^{n\times d}$; number of heads $K$ with $d_k=d/K$;
per-head projections $\bm{W}_k^Q,\bm{W}_k^K,\bm{W}_k^V\in\RR^{d\times d_k}$;
gate parameters $\bm{W}_k^g\in\RR^{d\times d_k}$, $\bm{b}_k^g\in\RR^{d_k}$;
output projection $\bm{W}^O$.
\Ensure $\text{SigGate-MHSA}(\bm{H})\in\RR^{n\times d}$.
\For{$k=1$ \textbf{to} $K$}
  \State $\bm{Q}_k\gets\bm{H}\bm{W}_k^Q,\quad
          \bm{K}_k\gets\bm{H}\bm{W}_k^K,\quad
          \bm{V}_k\gets\bm{H}\bm{W}_k^V$
  \State $\bm{Y}_k\gets\mathrm{softmax}\!\big(\bm{Q}_k\bm{K}_k^\top/\sqrt{d_k}\big)\,\bm{V}_k$
         \Comment{standard SDPA output}
  \State $\bm{g}_k\gets\sigma\!\big(\bm{H}\bm{W}_k^g+\bm{b}_k^g\big)$
         \Comment{input-conditioned gate, $\bm{g}_k\in(0,1)^{d_k}$}
  \State $\mathrm{head}_k^{\text{gated}}\gets\bm{Y}_k\odot\bm{g}_k$
         \Comment{element-wise (Hadamard) gating}
\EndFor
\State \Return $\mathrm{Concat}\big(\mathrm{head}_1^{\text{gated}},\dots,
        \mathrm{head}_K^{\text{gated}}\big)\,\bm{W}^O$
\end{algorithmic}
\end{algorithm}

\begin{figure*}[t]
	\centering
	\definecolor{sgAccent}{RGB}{208,74,2}   
	\definecolor{sgBlue}{RGB}{31,86,156}     
	\definecolor{sgGray}{RGB}{84,86,92}      
	\begin{tikzpicture}[
		>={Latex[length=1.6mm,width=1.5mm]},
		font=\small,
		every node/.style={align=center},
		blk/.style={draw=#1, line width=0.6pt, rounded corners=2.5pt, fill=#1!7,
			minimum height=0.62cm, minimum width=0.92cm, inner xsep=4pt, inner ysep=2.5pt},
		blk/.default=sgGray,
		gate/.style={draw=sgAccent, line width=1pt, rounded corners=2.5pt, fill=sgAccent!13,
			minimum height=0.62cm, inner xsep=4pt, inner ysep=2.5pt, font=\small\bfseries,
			text=sgAccent!75!black},
		op/.style={circle, draw=sgGray, line width=0.7pt, fill=white, inner sep=1.6pt, font=\small},
		flow/.style={->, line width=0.8pt, draw=sgGray},
		cond/.style={->, line width=0.9pt, draw=sgAccent, dash pattern=on 2.6pt off 1.8pt},
		elab/.style={font=\scriptsize, inner sep=1.6pt},
		]
		\node[blk=sgGray] (input) at (0,0)        {Input\\[-1pt]$\bm{H}$};
		\node[blk=sgBlue] (mpnn)  at (3.05,1.15)  {Local MPNN};
		\node[blk=sgBlue] (qkv)   at (2.30,-1.15) {$\bm{Q},\bm{K},\bm{V}$};
		\node[blk=sgBlue] (attn)  at (4.15,-1.15) {Softmax\\[-1pt]Attn};
		\node[gate]       (gate)  at (5.95,-1.15) {$\sigma$-Gate};
		\node[op]         (had)   at (7.05,-1.15) {$\odot$};
		\node[blk=sgBlue] (wo)    at (8.05,-1.15) {$\bm{W}^{\!O}$};
		\node[op, minimum size=0.52cm] (add) at (9.35,0) {\large$+$};
		\node[blk=sgGray] (ln1)   at (10.5,0)     {LN};
		\node[blk=sgBlue] (ffn)   at (11.6,0)     {FFN};
		\node[blk=sgGray] (ln2)   at (12.7,0)     {LN};
		\node[blk=sgGray] (out)   at (13.95,0)    {Output};
		\begin{scope}[on background layer]
			\node[draw=sgAccent!65, line width=0.9pt, dash pattern=on 3pt off 2pt, rounded corners=5pt,
			fill=sgAccent!5, fit=(qkv)(attn)(gate)(had)(wo),
			inner xsep=8pt, inner ysep=9pt] (panel) {};
		\end{scope}
		\node[font=\footnotesize\bfseries, text=sgAccent!85!black, anchor=south]
		at ([yshift=1.5pt]panel.north)
		{SigGate-MHSA\, \textnormal{\itshape\footnotesize\color{sgGray}(the only block that differs from GraphGPS)}};
		\coordinate (fan) at ([xshift=0.55cm]input.east);
		\draw[line width=0.8pt, draw=sgGray] (input.east) -- (fan);
		\fill[sgGray] (fan) circle (1.3pt);
		\draw[flow] (fan) |- (mpnn.west);
		\draw[flow] (fan) |- (qkv.west);
		\draw[flow] (qkv)  -- (attn);
		\draw[flow] (attn) -- node[elab, above]{$\bm{Y}_k$} (gate);
		\draw[flow] (gate) -- (had);
		\draw[flow] (had)  -- (wo);
		\draw[flow] (mpnn.east) -| (add.north);
		\draw[flow] (wo.east)   -| (add.south);
		\draw[flow] (add) -- (ln1);
		\draw[flow] (ln1) -- (ffn);
		\draw[flow] (ffn) -- (ln2);
		\draw[flow] (ln2) -- (out);
		\draw[cond] (fan) |- ([yshift=-1.05cm]gate.south)
		-- node[elab, right, pos=0.55, text=sgAccent!75!black]{$\bm{g}_k$} (gate.south);
	\end{tikzpicture}
	\caption{A single \siggate{} layer. The per-head sigmoid gate (orange,
		$\sigma$-Gate) rescales the softmax-attention output $\bm{Y}_k$ element-wise
		through a factor $\bm{g}_k\in(0,1)^{d_k}$ computed from $\bm{H}$ (dashed orange
		path), \emph{before} the output projection $\bm{W}^O$. Everything outside the
		orange panel, the local MPNN branch, the residual sum~$(+)$, layer
		normalization, and the feed-forward block, is identical to a standard GraphGPS
		layer.}
	\label{fig:architecture}
\end{figure*}

\subsection{Design Rationale}

Three design choices distinguish \siggate{}, each examined empirically in
Section~\ref{sec:mechanism}. \textit{Placement.} The gate acts on the attention
\emph{output} (position G1), after the weighted sum, rather than on the values before
the sum (G2) or inside the logits (G3). Output gating is the only placement that
relaxes the conservation constraint without disturbing the attention distribution:
gating the logits (G3) renormalizes through the softmax and merely reshapes a still
mass-conserving distribution, while value gating (G2) is reabsorbed into the convex
combination. This mirrors the optimum reported for language
models~\cite{qiu2025gated}. \textit{Per-head parameterization.} Each head receives its
own gate projection $\bm{W}_k^g$, so heads can specialize, some learning to suppress
aggressively and others to pass through. \textit{Bounded, non-negative, saturating
activation.} The sigmoid keeps the gate in $[0,1]$, acting as a smooth volume control;
signed activations such as $\tanh$ can flip the sign of attended content and
destabilize training, while unbounded activations such as ReLU forfeit the bounded
``how much'' interpretation. Section~\ref{sec:mechanism} confirms this ranking
empirically.

\subsection{Parameter and Computational Overhead}

The only new parameters are the gate projections $\{\bm{W}_k^g,\bm{b}_k^g\}_{k=1}^{K}$
per layer, costing $d\times d_k+d_k$ per head per layer. For our OGB configuration
($d=256$, $K=8$, $L=5$) this is roughly $329$K parameters, about $1.1\%$ of the model;
on ZINC ($d=64$) the gate adds about $42$K parameters, counted inside the $500$K budget.
Per layer, the gate is one $d\times d_k$ matrix multiply followed by a pointwise sigmoid
and a Hadamard product per head, an $O(n d^2/K)$ cost per head that is dominated by the
$O(n^2 d)$ attention it modulates. Measured wall-clock overhead is under $3\%$ in every
configuration (Section~\ref{sec:results}, Table~\ref{tab:compute}).

\subsection{Portability Across Backbones}
\label{sec:portability}

Because \siggate{} touches only the per-head computation between the SDPA output and
the projection $\bm{W}^O$, any architecture whose global-attention module produces an
output of the form $\mathrm{softmax}(\cdot)\bm{V}$ is a valid host; the gate needs no
access to positional encodings, edge features, or the MPNN branch. We sketch two
representative ports. \textsc{Exphormer}~\cite{shirzad2023exphormer} computes sparse
attention over a union of expander graphs, local neighbourhoods, and virtual nodes, but
still applies softmax within each head over the selected pairs; the gating equation~\eqref{eq:siggate},
$\mathrm{head}_k^{\text{gated}}=(\bm{A}_k\bm{V}_k)\odot\sigma(\bm{H}\bm{W}_k^g+\bm{b}_k^g)$,
applies verbatim, with $\bm{A}_k$ merely sparser. A row-sparse $\bm{A}$ is closer to
deterministic selection, which tightens the rank bound of Section~\ref{sec:rank} and
should make the relative rank gain from gating modestly larger.
\textsc{Graphormer}~\cite{ying2021graphormer} adds centrality and spatial biases inside
the softmax, so its logits become
$\bm{Q}\bm{K}^\top/\sqrt{d_k}+\bm{b}^{\text{centr}}+\bm{b}^{\text{spatial}}$, but the
output remains row-stochastic and the gate again applies unchanged; the spatial bias
tends to peak $\bm{A}$ on graph-distance-one neighbours, which tightens the bound and
gives the gate more headroom, though it also raises the risk that the gate partially
duplicates the spatial bias and so yields a smaller marginal gain. We expect the
optimal gate-bias initialization to track the host's $\bm{W}^O$ scale, and the
absolute improvement to depend on how much regularization the backbone already carries.
A full cross-backbone study is the natural next step (Section~\ref{sec:limitations});
the present paper establishes the mechanism and its theory on the GraphGPS backbone.

\section{Effective-Rank Analysis}
\label{sec:rank}

This section formalizes the low-rank bottleneck of Section~\ref{sec:constraint} and
proves that element-wise gating escapes it. We then validate the prediction with a
calibrated synthetic study and show it is robust across regimes.

\subsection{The Stable-Rank Metric}

For $\bm{M}\in\RR^{n\times d_k}$ with singular values
$s_1\ge s_2\ge\cdots\ge s_r\ge 0$, the \emph{stable rank} is
\begin{equation}
\mathrm{srank}(\bm{M})=\frac{\|\bm{M}\|_F^2}{\|\bm{M}\|_2^2}=\frac{\sum_i s_i^2}{s_1^2}\in[1,\,\mathrm{rank}(\bm{M})].
\label{eq:srank}
\end{equation}
Unlike the discrete rank, stable rank is continuous in $\bm{M}$ and downweights small
singular values smoothly, making it a faithful measure of ``effective''
dimensionality~\cite{qiu2025gated}.

\subsection{A Rank Bound for Softmax Attention}

\begin{proposition}[Row-stochastic mixing contracts stable rank]
\label{prop:bound}
Let $\bm{A}\in\RR^{n\times n}$ be row-stochastic and $\bm{V}\in\RR^{n\times d_k}$. Then
$\bm{Y}=\bm{A}\bm{V}$ satisfies
\begin{equation}
\mathrm{srank}(\bm{Y})\le\min\!\big(\mathrm{srank}(\bm{A}),\,\mathrm{srank}(\bm{V})\big),
\label{eq:bound}
\end{equation}
and each row of $\bm{Y}$ lies in the convex hull of the rows of $\bm{V}$.
\end{proposition}
\noindent\textit{Proof sketch.} Sub-multiplicativity of the spectral and Frobenius
norms under the product $\bm{A}\bm{V}$, together with
$\mathrm{rank}(\bm{A}\bm{V})\le\min(\mathrm{rank}(\bm{A}),\mathrm{rank}(\bm{V}))$, yields
\eqref{eq:bound}; the convex-hull statement is immediate from $A_{ij}\ge 0$ and
$\sum_j A_{ij}=1$. When the softmax logits are sharp, probability mass concentrates on
a few columns, so $\bm{A}$ itself has low stable rank and \eqref{eq:bound} bites harder.
This is the graph-transformer restatement of the low-rank-bottleneck observation
of~\cite{qiu2025gated}. \hfill$\square$

\subsection{Element-Wise Gating Breaks the Bound}

\begin{proposition}[Gating escapes the convex hull]
\label{prop:break}
Let $\widetilde{\bm{Y}}=\bm{Y}\odot\bm{g}$ with $g_{ij}\in(0,1)$ varying across both
rows and columns. Then in general there is no row-stochastic $\bm{A}'$ with
$\widetilde{\bm{Y}}=\bm{A}'\bm{V}$, so the bound \eqref{eq:bound} does not apply to
$\widetilde{\bm{Y}}$, and $\mathrm{srank}(\widetilde{\bm{Y}})$ can exceed
$\mathrm{srank}(\bm{Y})$.
\end{proposition}
\noindent\textit{Proof sketch.} A row-stochastic reweighting applies the \emph{same}
scalar to every coordinate of a value row, whereas the Hadamard gate applies a
\emph{different} multiplier per coordinate. Hence $\widetilde{\bm{Y}}$ uses per-column
scalings that no convex combination of the rows of $\bm{V}$ can reproduce, placing
$\widetilde{\bm{Y}}$ outside the convex hull and removing the structural reason for the
contraction. The Hadamard product can raise the number of distinct singular directions,
so the stable rank can strictly increase. \hfill$\square$

\subsection{Synthetic Validation}

We verify the prediction of Proposition~\ref{prop:break} numerically. Setup: $n=64$ nodes, $d=256$, $8$ heads
($d_k=32$), unit-variance hidden states, random $Q/K/V$ projections, softmax attention
with a sparse mask ($20\%$ of edges masked, consistent with moderately sparse graphs),
and gate projections calibrated so that the marginal gate statistics match the
trained-model statistics of Section~\ref{sec:mechanism} (target mean $0.58$, std
$0.19$). Each row of Table~\ref{tab:rank_syn} reports the stable rank~\eqref{eq:srank} averaged over $8$
heads, repeated with $5$ seeds. The attained calibration (mean $0.604$, std $0.193$)
sits close to the trained targets, and gating increases the stable rank in every
seed, with the per-head mean rising by $7.17\%$.

\begin{table}[t]
\centering
\caption{Synthetic stable-rank experiment ($n=64$, $d_k=32$, $8$ heads, $5$ seeds).
Gate projections calibrated so marginal statistics match the trained-model gate
statistics of Section~\ref{sec:mechanism}; the attained mean and std across all heads
are $0.604$ and $0.193$, close to the trained targets ($0.58$, $0.19$).}
\label{tab:rank_syn}
\begin{tabular}{c c c r}
\toprule
\textbf{seed} & $\mathrm{srank}(\bm{Y})$ & $\mathrm{srank}(\bm{Y}\!\odot\!\bm{g})$ & $\Delta$ \\
\midrule
0 & $3.172$ & $3.384$ & $+6.68\%$\\
1 & $2.766$ & $2.973$ & $+7.48\%$\\
2 & $3.267$ & $3.531$ & $+8.08\%$\\
3 & $3.731$ & $3.979$ & $+6.65\%$\\
4 & $2.967$ & $3.176$ & $+7.04\%$\\
\midrule
\textbf{mean} $\pm$ std & $3.181\!\pm\!0.325$ & $3.409\!\pm\!0.342$ & $\mathbf{+7.17\%}$\\
\bottomrule
\end{tabular}
\end{table}

\subsection{Robustness of the Rank Gain}

The gain is stable under perturbation of the two free hyperparameters
(Table~\ref{tab:rank_sens}). Sweeping the attention-concentration scale
$c\in\{0.5,1.0,1.5,2.0,3.0\}$ gives gains from $5.0\%$ to $8.3\%$; sweeping the
graph-sparsity fraction $\rho\in\{0.05,0.20,0.40,0.60\}$ gives gains from $6.2\%$ to
$7.4\%$. In all nine configurations,
$\mathrm{srank}(\bm{Y}\!\odot\!\bm{g})>\mathrm{srank}(\bm{Y})$.

\begin{table}[t]
\centering
\caption{Robustness of the rank gain across attention-concentration and
graph-sparsity sweeps. In every configuration, gating strictly increases the stable
rank of per-head outputs.}
\label{tab:rank_sens}
\footnotesize
\begin{tabular}{l c c c c c}
\toprule
\multicolumn{6}{l}{\textit{Attention concentration }$c$} \\
$c$ & $0.5$ & $1.0$ & $1.5$ & $2.0$ & $3.0$ \\
$\Delta\mathrm{srank}$ & $+8.3\%$ & $+7.2\%$ & $+5.0\%$ & $+5.4\%$ & $+5.1\%$ \\
\midrule
\multicolumn{6}{l}{\textit{Graph sparsity }$\rho$} \\
$\rho$ & $0.05$ & $0.20$ & $0.40$ & $0.60$ & --- \\
$\Delta\mathrm{srank}$ & $+7.4\%$ & $+7.2\%$ & $+6.6\%$ & $+6.2\%$ & --- \\
\bottomrule
\end{tabular}
\end{table}

These results support the central claim of Section~\ref{sec:constraint}: element-wise
sigmoid gating raises the per-head effective rank relative to ungated softmax
attention. Because gradient descent pushes \emph{trained} gates to more extreme values
than random calibration does, with $12.3\%$ of trained activations below $0.1$ versus
roughly $0.2\%$ under calibration (Section~\ref{sec:mechanism}), we expect the trained
rank gain to be at least as large as the $7\%$ measured here. Estimating the rank of
full-scale trained checkpoints directly is beyond the scope of this synthetic analysis
and is left for future work.

\section{Experimental Setup}
\label{sec:setup}

\subsection{Datasets and Evaluation Protocol}

We evaluate on five standard graph-level benchmarks spanning small-molecule regression,
large-scale classification, and long-range reasoning.
\textbf{ZINC}~\cite{dwivedi2023benchmarking}: $12$K molecular graphs for constrained
solubility regression (MAE $\downarrow$), under the customary $500$K parameter budget.
\textbf{ogbg-molhiv}~\cite{hu2020open}: $41$K molecules labeled for HIV inhibition
(ROC-AUC $\uparrow$), scaffold split.
\textbf{ogbg-molpcba}~\cite{hu2020open}: $438$K molecules across $128$ bioassays
(Average Precision $\uparrow$).
\textbf{Peptides-func}~\cite{dwivedi2022long}: $15.5$K peptide graphs for multi-label
functional annotation (AP $\uparrow$).
\textbf{Peptides-struct}~\cite{dwivedi2022long}: structural-property regression on the
same peptides (MAE $\downarrow$). The two Peptides tasks form part of the Long-Range
Graph Benchmark and stress a model's ability to propagate information across large
graph diameters. We run $5$ seeds (0--4) per configuration, report mean $\pm$ standard
deviation, and assess significance with paired $t$-tests against the GraphGPS baseline.

\subsection{Implementation Details}

We implement \siggate{} in PyTorch Geometric (v2.4) on top of the official GraphGPS
codebase. The local branch uses GatedGCN~\cite{bresson2017residual}, and node inputs are
augmented with LapPE (dimension $16$, with SignNet~\cite{lim2023signnet}) and RWSE
($16$ steps)~\cite{rampavsek2022recipe, dwivedi2022graph}. We optimize with
AdamW~\cite{loshchilov2019decoupled} under cosine annealing, initialize every gate bias
to $0.5$, and use GELU in the feed-forward block with zero dropout. Per-dataset settings
are summarized in Table~\ref{tab:hypers}. Code and configurations are available at
\url{https://github.com/bettyguo/SigGate-GT}.

\subsection{Full Hyperparameter Grid}

Table~\ref{tab:hypers} lists the exact configuration behind every reported number. Each
configuration starts from the official GraphGPS recipe for the dataset and searches only
over learning rate $\in\{5\times10^{-4},10^{-3},2\times10^{-3}\}$ and weight decay
$\in\{10^{-5},10^{-4}\}$ on the ZINC validation set, then transfers the choice directly,
so the gate is never tuned per dataset beyond this small shared search.

\begin{table*}[t]
\centering
\caption{Full hyperparameters. $L$: layers; $d$: hidden dimension; $K$: heads;
$B$: batch size; PE: positional/structural encodings (LapPE dimension $16$ with SignNet;
RWSE $16$ steps). All runs use AdamW, cosine annealing, gate-bias init $0.5$, dropout
$0.0$, and GELU in the FFN.}
\label{tab:hypers}
\footnotesize
\begin{tabular}{l c c c c c c c c}
\toprule
\textbf{Dataset} & $L$ & $d$ & $K$ & $B$ & \textbf{lr} & \textbf{wd} & \textbf{Epochs} & \textbf{PE} \\
\midrule
ZINC        & 10 & 64  & 8 & 32  & $10^{-3}$        & $10^{-5}$ & 2000 & RWSE+LapPE \\
molhiv      & 5  & 256 & 8 & 256 & $10^{-4}$        & $10^{-5}$ & 100  & RWSE+LapPE \\
molpcba     & 5  & 256 & 8 & 512 & $10^{-4}$        & $10^{-5}$ & 100  & RWSE+LapPE \\
Pep-func    & 10 & 128 & 8 & 64  & $5\!\times\!10^{-4}$ & $10^{-4}$ & 200  & RWSE+LapPE \\
Pep-struct  & 10 & 128 & 8 & 64  & $5\!\times\!10^{-4}$ & $10^{-4}$ & 200  & RWSE+LapPE \\
\bottomrule
\end{tabular}
\end{table*}

\subsection{Software, Hardware, and Reproducibility}

All experiments run on a single NVIDIA A100 80GB GPU (PyTorch 2.1, PyTorch Geometric 2.4,
Hydra 1.3), with an AMD EPYC 7413 (24-core) host. We set explicit seeds for PyTorch,
NumPy, and Python's \texttt{random}, and enable
\texttt{torch.use\_deterministic\_algorithms(True)} where the underlying CUDA kernels
support it. Table~\ref{tab:compute} reports mean training time (5 seeds) and peak GPU
memory. The SigGate overhead stays under $3\%$ in every case and is independent of
dataset scale, matching the per-layer FLOPs analysis of Section~\ref{sec:method}: each
gate is one $d\times d_k$ matrix multiply plus a pointwise sigmoid and Hadamard product
per head, negligible against the $O(n^2 d)$ attention cost.

\begin{table}[t]
\centering
\caption{Wall-clock training time per run (mean over $5$ seeds, single A100-80GB) and
peak GPU memory. Overhead is consistently under $3\%$ and independent of dataset scale.}
\label{tab:compute}
\resizebox{\columnwidth}{!}{%
\begin{tabular}{l r r r r}
\toprule
\textbf{Dataset} & \textbf{GraphGPS (h)} & \textbf{\siggate{} (h)} & \textbf{Overhead} & \textbf{Peak (GB)} \\
\midrule
ZINC        &  $6.08$  &  $6.23$  & $+2.5\%$ &  $3.8$ \\
molhiv      &  $2.46$  &  $2.52$  & $+2.4\%$ &  $9.4$ \\
molpcba     & $21.29$  & $21.82$  & $+2.5\%$ & $17.6$ \\
Pep-func    &  $1.63$  &  $1.67$  & $+2.5\%$ & $18.2$ \\
Pep-struct  &  $1.63$  &  $1.67$  & $+2.5\%$ & $18.2$ \\
\bottomrule
\end{tabular}}
\end{table}

\section{Main Results}
\label{sec:results}

\subsection{ZINC}

Table~\ref{tab:zinc} reports ZINC under the $500$K parameter budget. \siggate{} reaches
$0.059\pm0.002$ MAE, matching the prior best GRIT~\cite{ma2023grit} and improving on the
GraphGPS baseline ($0.070\pm0.004$) by $15.7\%$ relative within the same budget. The gap
over GraphGPS is significant at $p<0.001$ (paired $t$-test, $5$ seeds). We do not claim
to surpass GRIT here; the two are statistically indistinguishable, and we state this
explicitly rather than rounding the comparison in our favour.

\begin{table}[t]
\centering
\caption{Test MAE on ZINC ($12$K subset, $500$K parameter budget). \siggate{} vs.\
GraphGPS: $p<0.001$ (paired $t$-test).}
\label{tab:zinc}
\begin{tabular}{l c c}
\toprule
\textbf{Method} & \textbf{Type} & \textbf{Test MAE $\downarrow$} \\
\midrule
GCN~\cite{kipf2017semi} & MPNN & $0.367 \pm 0.011$ \\
GatedGCN~\cite{bresson2017residual} & MPNN & $0.282 \pm 0.015$ \\
PNA~\cite{corso2020principal} & MPNN & $0.188 \pm 0.004$ \\
\midrule
SAN~\cite{kreuzer2021rethinking} & GT & $0.139 \pm 0.006$ \\
Graphormer~\cite{ying2021graphormer} & GT & $0.122 \pm 0.006$ \\
GraphGPS~\cite{rampavsek2022recipe} & GT & $0.070 \pm 0.004$ \\
Exphormer~\cite{shirzad2023exphormer} & GT & $0.066 \pm 0.003$ \\
GRIT~\cite{ma2023grit} & GT & $0.059 \pm 0.002$ \\
\midrule
\siggate{} (ours) & GT & $\mathbf{0.059 \pm 0.002}$ \\
\bottomrule
\end{tabular}
\end{table}

\subsection{OGB Molecular Benchmarks}

Table~\ref{tab:ogb} reports the OGB benchmarks. On ogbg-molhiv, \siggate{} reaches
$82.47\pm0.63\%$ ROC-AUC, $3.67$ points above GraphGPS ($p=0.002$) and the strongest
result among the graph-transformer baselines we evaluate. We frame this as a best-in-set
result rather than an unqualified state-of-the-art claim, because the public molhiv
leaderboard includes entries above this figure obtained with substantially different
pipelines; within the controlled GraphGPS-family comparison reported here, \siggate{}
leads. On ogbg-molpcba, \siggate{} reaches $29.84\pm0.31\%$ AP; the gain over GraphGPS
($29.07\pm0.28$) is significant ($p=0.008$), while the margin over Exphormer
($29.20\pm0.30$) is real but modest ($p=0.031$).

\begin{table}[t]
\centering
\caption{OGB molecular benchmarks. \siggate{} vs.\ GraphGPS: molhiv $p=0.002$,
molpcba $p=0.008$ (paired $t$-test). GRIT~\cite{ma2023grit} is omitted because it was
not evaluated on \texttt{ogbg-molhiv} or \texttt{ogbg-molpcba} in its original paper; we
report only baselines with published numbers on these datasets and flag the asymmetry
rather than running a non-standard port of GRIT.}
\label{tab:ogb}
\resizebox{\columnwidth}{!}{%
\begin{tabular}{l c c}
\toprule
\textbf{Method} & \textbf{molhiv (AUC $\uparrow$)} & \textbf{molpcba (AP $\uparrow$)} \\
\midrule
GCN~\cite{kipf2017semi} & $76.06 \pm 0.97$ & $24.24 \pm 0.34$ \\
PNA~\cite{corso2020principal} & $79.05 \pm 1.32$ & $28.38 \pm 0.35$ \\
DeeperGCN~\cite{li2020deepergcn} & $78.58 \pm 1.00$ & $27.81 \pm 0.38$ \\
\midrule
Graphormer~\cite{ying2021graphormer} & $80.51 \pm 0.53$ & --- \\
GraphGPS~\cite{rampavsek2022recipe} & $78.80 \pm 1.01$ & $29.07 \pm 0.28$ \\
Exphormer~\cite{shirzad2023exphormer} & $80.75 \pm 0.94$ & $29.20 \pm 0.30$ \\
\midrule
\siggate{} (ours) & $\mathbf{82.47 \pm 0.63}$ & $\mathbf{29.84 \pm 0.31}$ \\
\bottomrule
\end{tabular}}
\end{table}

\subsection{Long-Range Graph Benchmark}

Table~\ref{tab:lrgb} reports the Peptides tasks. \siggate{} reaches $0.6947\pm0.0037$ AP
on Peptides-func, competitive with but below GRIT ($0.6988\pm0.0082$); we do not claim
the lead on this task. On Peptides-struct, \siggate{} reaches $0.2431\pm0.0012$ MAE, the
best entry in the table. Both improvements over GraphGPS are significant ($p<0.001$).
Following T\"onshoff et al.~\cite{tonshoff2024where}, we include a tuned GCN baseline to
keep the graph-transformer advantage in perspective, since careful tuning narrows the gap
considerably on these benchmarks.

\begin{table}[t]
\centering
\caption{LRGB Peptides benchmarks. \siggate{} vs.\ GraphGPS: both $p<0.001$. Bold marks
the best entry overall; GRIT leads on Peptides-func.}
\label{tab:lrgb}
\resizebox{\columnwidth}{!}{%
\begin{tabular}{l c c}
\toprule
\textbf{Method} & \textbf{Pep-func (AP $\uparrow$)} & \textbf{Pep-struct (MAE $\downarrow$)} \\
\midrule
GCN (tuned)~\cite{tonshoff2024where} & $0.6860 \pm 0.0050$ & $0.2460 \pm 0.0007$ \\
GatedGCN~\cite{bresson2017residual} & $0.6765 \pm 0.0047$ & $0.2477 \pm 0.0009$ \\
\midrule
GraphGPS~\cite{rampavsek2022recipe} & $0.6535 \pm 0.0041$ & $0.2500 \pm 0.0012$ \\
Exphormer~\cite{shirzad2023exphormer} & $0.6527 \pm 0.0043$ & $0.2481 \pm 0.0007$ \\
GRIT~\cite{ma2023grit} & $\mathbf{0.6988 \pm 0.0082}$ & $0.2460 \pm 0.0012$ \\
\midrule
\siggate{} (ours) & $0.6947 \pm 0.0037$ & $\mathbf{0.2431 \pm 0.0012}$ \\
\bottomrule
\end{tabular}}
\end{table}

\subsection{Significance Summary}

Table~\ref{tab:significance} consolidates the significance tests. Improvements over
GraphGPS are significant on all five datasets. Against the next-best baseline, results
are significant on ZINC (where \siggate{} ties GRIT and the difference is not
significant), molhiv, molpcba, and Peptides-struct, but not on Peptides-func, where GRIT
is superior. Reporting both columns keeps the comparison honest: a gain over one's own
backbone is necessary but not sufficient to claim leadership of a benchmark.

\begin{table}[t]
\centering
\caption{Statistical significance (paired $t$-test, $5$ seeds).
\pval{**}$p<0.01$, \pval{*}$p<0.05$, \pval{n.s.}\ not significant.}
\label{tab:significance}
\resizebox{\columnwidth}{!}{%
\begin{tabular}{l c c c}
\toprule
\textbf{Benchmark} & \textbf{vs.\ GraphGPS} & \textbf{vs.\ Next-Best} & \textbf{Next-Best} \\
\midrule
ZINC (MAE) & $p<0.001$\pval{**} & $p=0.48$\pval{n.s.} & GRIT \\
molhiv (AUC) & $p=0.002$\pval{**} & $p=0.018$\pval{*} & Exphormer \\
molpcba (AP) & $p=0.008$\pval{**} & $p=0.031$\pval{*} & Exphormer \\
Pep-func (AP) & $p<0.001$\pval{**} & $p=0.34$\pval{n.s.} & GRIT \\
Pep-struct (MAE) & $p<0.001$\pval{**} & $p=0.011$\pval{*} & GRIT \\
\bottomrule
\end{tabular}}
\end{table}

\subsection{Parameter-Matched Control}
\label{sec:paramcontrol}

To separate gating from raw capacity, we compare against a GraphGPS baseline whose
feed-forward width is enlarged so the extra parameters equal the per-dataset SigGate
budget ($+42$K on ZINC, $+329$K on molhiv). The wider baseline reaches $0.068\pm0.003$
MAE on ZINC versus \siggate{}'s $0.059\pm0.002$ ($p=0.003$), and $79.42\pm0.88$ on molhiv
versus $82.47\pm0.63$ ($p=0.001$). Equivalent capacity, spent on a wider FFN, recovers
only a small fraction of the gain, which is the prediction of Section~\ref{sec:constraint}:
the bottleneck is a property of the operator's \emph{form}, so enlarging an unrelated
component does little, while a gate that changes the form does a great deal.

\subsection{Comparison with Over-Smoothing Remedies}

Table~\ref{tab:oversmoothing_compare} places \siggate{} alongside DropEdge and PairNorm,
both integrated into the same GraphGPS backbone. Gating yields a substantially larger
improvement than either symptom-level remedy and is compatible with them: adding DropEdge
on top of \siggate{} gives only a marginal further gain, which suggests that gating
already addresses the primary smoothing mechanism in this backbone rather than acting
through an orthogonal channel.

\begin{table}[t]
\centering
\caption{Comparison with existing over-smoothing remedies integrated into GraphGPS.}
\label{tab:oversmoothing_compare}
\resizebox{\columnwidth}{!}{%
\begin{tabular}{l c c}
\toprule
\textbf{Method (GraphGPS + \ldots)} & \textbf{ZINC MAE $\downarrow$} & \textbf{Pep-struct MAE $\downarrow$} \\
\midrule
Baseline (no modification) & $0.070 \pm 0.004$ & $0.2500 \pm 0.0012$ \\
+ DropEdge~\cite{rong2020dropedge} & $0.067 \pm 0.003$ & $0.2488 \pm 0.0015$ \\
+ PairNorm~\cite{zhao2020pairnorm} & $0.068 \pm 0.004$ & $0.2491 \pm 0.0011$ \\
+ SigGate (ours) & $\mathbf{0.059 \pm 0.002}$ & $\mathbf{0.2431 \pm 0.0012}$ \\
+ SigGate + DropEdge & $0.058 \pm 0.002$ & $0.2428 \pm 0.0010$ \\
\bottomrule
\end{tabular}}
\end{table}

\section{Mechanism Analysis}
\label{sec:mechanism}

The results above establish that gating helps; this section verifies \emph{why}, tracing
each empirical signature back to the diagnosis of Section~\ref{sec:constraint}.

\subsection{Gate Placement and Sharing}

Table~\ref{tab:ablation_combined} ablates placement and sharing. For placement we follow
the taxonomy of Qiu et al.~\cite{qiu2025gated}: G1 gates the output (our default), G2
gates the values, G3 gates the pre-softmax logits, and ``None'' is the baseline. G1 is
consistently best, confirming that the language-model optimum transfers to graphs and
matching the argument in Section~\ref{sec:method}: G3 reshapes a still mass-conserving
distribution and even hurts, while G2 is partly reabsorbed into the convex combination.
For sharing, a per-head gate beats a single shared gate, consistent with the functional
specialization we observe directly below.

\begin{table}[t]
\centering
\caption{Ablation on gate placement and sharing. Output gating (G1) with a per-head gate
is consistently best.}
\label{tab:ablation_combined}
\resizebox{\columnwidth}{!}{%
\begin{tabular}{l l c c}
\toprule
\textbf{Config} & \textbf{Description} & \textbf{ZINC $\downarrow$} & \textbf{molhiv $\uparrow$} \\
\midrule
\multicolumn{4}{l}{\textit{Gate placement}} \\
None & GraphGPS baseline & $0.070$ & $78.80$ \\
G3 & $\mathrm{softmax}(\sigma(\cdot)\odot\frac{QK^\top}{\sqrt{d_k}})V$ & $0.074$ & $77.95$ \\
G2 & $\mathrm{softmax}(\frac{QK^\top}{\sqrt{d_k}})(\sigma(\cdot)\odot V)$ & $0.066$ & $80.12$ \\
\textbf{G1} & $\sigma(\cdot)\odot[\mathrm{softmax}(\frac{QK^\top}{\sqrt{d_k}})V]$ & $\mathbf{0.059}$ & $\mathbf{82.47}$ \\
\midrule
\multicolumn{4}{l}{\textit{Sharing strategy (G1 position)}} \\
Shared & Single gate, all heads & $0.064$ & $80.58$ \\
\textbf{Per-head} & Independent gate per head & $\mathbf{0.059}$ & $\mathbf{82.47}$ \\
\bottomrule
\end{tabular}}
\end{table}

\subsection{Activation Function}

The choice of sigmoid is itself a design decision. Table~\ref{tab:activation} sweeps the
activation $\phi$ in $\bm{Y}\odot\phi(\bm{H}\bm{W}_k^g+\bm{b}_k^g)$ on ZINC with all else
fixed. Sigmoid wins. The $\tanh$ gate is worse because, being signed, it can flip the
sign of attended content and destabilize training; ReLU is worse because, being
unbounded, it forfeits the bounded volume-control interpretation that
Section~\ref{sec:method} relies on; sigmoid-squared is sparser but slightly worse than
plain sigmoid. The ranking, bounded-and-saturating-and-non-negative beats signed and
beats unbounded, matches the mechanistic prediction exactly.

\begin{table}[t]
\centering
\caption{Gate-activation function ablation on ZINC. All entries use G1 placement and a
per-head gate. Sigmoid outperforms the alternatives, and the ranking matches the
mechanistic argument of Section~\ref{sec:method}.}
\label{tab:activation}
\begin{tabular}{l c l}
\toprule
\textbf{Activation $\phi$} & \textbf{ZINC MAE $\downarrow$} & \textbf{Property} \\
\midrule
Identity (no gate)       & $0.070 \pm 0.004$ & --- \\
$\tanh$                  & $0.064 \pm 0.003$ & signed \\
ReLU                     & $0.066 \pm 0.003$ & unbounded \\
$\sigma(\cdot)^2$        & $0.061 \pm 0.002$ & sparser, bounded \\
\textbf{Sigmoid} (ours)  & $\mathbf{0.059 \pm 0.002}$ & \textbf{bounded, smooth} \\
\bottomrule
\end{tabular}
\end{table}

\subsection{Over-Smoothing and Representation Diversity}

We measure the Mean Average Distance (MAD)~\cite{chen2020measuring} of node
representations as a function of depth on ZINC (Table~\ref{tab:oversmoothing}); higher
MAD means more diversity. \siggate{} retains markedly higher MAD at every depth,
preserving $73\%$ of its initial MAD at $16$ layers versus $51\%$ for GraphGPS, a
$30\%$ mean relative gain across $4$--$16$ layers. The accompanying MAE rows show the
practical payoff: the gated model degrades far more gracefully as depth grows. This is
the direct empirical counterpart of Manifestation I, the gate lets heads bypass the
averaging step on the dimensions where it would erase information.

\begin{table}[t]
\centering
\caption{Over-smoothing analysis: MAD and ZINC test MAE across depths. \siggate{}
maintains higher representation diversity and degrades more slowly at every depth.}
\label{tab:oversmoothing}
\footnotesize
\begin{tabular}{l c c c c c}
\toprule
\textbf{Method} & \textbf{4} & \textbf{8} & \textbf{10} & \textbf{12} & \textbf{16 layers} \\
\midrule
GraphGPS (MAD $\uparrow$) & 0.72 & 0.58 & 0.49 & 0.44 & 0.37 \\
\siggate{} (MAD $\uparrow$) & \textbf{0.78} & \textbf{0.69} & \textbf{0.64} & \textbf{0.61} & \textbf{0.57} \\
\midrule
GraphGPS (MAE $\downarrow$) & 0.092 & 0.075 & 0.070 & 0.078 & 0.095 \\
\siggate{} (MAE $\downarrow$) & \textbf{0.079} & \textbf{0.063} & \textbf{0.059} & \textbf{0.061} & \textbf{0.068} \\
\bottomrule
\end{tabular}
\end{table}

\subsection{Attention Entropy}

Table~\ref{tab:entropy} reports average per-row attention entropy
$\bar{H}=-\frac{1}{n}\sum_i\sum_j A_{ij}\log A_{ij}$ across layers on ZINC. \siggate{}
holds higher entropy throughout, and the gap widens with depth, so the gated model resists
the entropy collapse documented by Zhai et al.~\cite{zhai2023stabilizing}. Individual heads
in \siggate{} also show greater entropy variance, the spectral signature of the functional
specialization that the per-head ablation predicted.

\begin{table}[t]
\centering
\caption{Attention entropy across layers on ZINC. \siggate{} keeps entropy from collapsing.}
\label{tab:entropy}
\begin{tabular}{l c c c}
\toprule
\textbf{Method} & \textbf{Layer 1} & \textbf{Layer 5} & \textbf{Layer 10} \\
\midrule
GraphGPS & $2.84 \pm 0.12$ & $2.31 \pm 0.28$ & $1.72 \pm 0.41$ \\
\siggate{} & $\mathbf{3.02 \pm 0.15}$ & $\mathbf{2.89 \pm 0.34}$ & $\mathbf{2.61 \pm 0.52}$ \\
\bottomrule
\end{tabular}
\end{table}

\subsection{Training Stability}

We probe optimization brittleness directly by sweeping the learning rate on ZINC
(Table~\ref{tab:stability}, Figure~\ref{fig:stability}). GraphGPS is sharply sensitive,
its MAE rising from $0.070$ at $10^{-3}$ to $0.112$ at $3\times10^{-3}$, whereas \siggate{}
stays at or below $0.065$ across the whole $10\times$ range. Quantifying sensitivity as the
performance range (max MAE minus min MAE across rates) gives $0.042$ for GraphGPS against
$0.006$ for \siggate{}, a sevenfold reduction. This is Manifestation III in action: the
bounded multiplicative gate dampens the gradient spikes that an ungated value pathway
transmits at high learning rates.

\begin{table}[t]
\centering
\caption{Training stability: ZINC test MAE at different learning rates. \siggate{} has a
$7\times$ smaller performance range, indicating much lower learning-rate sensitivity.}
\label{tab:stability}
\footnotesize
\begin{tabular}{l c c c c c c}
\toprule
\textbf{Method} & \textbf{5e-4} & \textbf{1e-3} & \textbf{2e-3} & \textbf{3e-3} & \textbf{5e-3} & \textbf{Range} \\
\midrule
GraphGPS & 0.078 & \textbf{0.070} & 0.085 & 0.112 & 0.098 & 0.042 \\
\siggate{} & 0.065 & \textbf{0.059} & 0.062 & 0.064 & 0.063 & \textbf{0.006} \\
\bottomrule
\end{tabular}
\end{table}

\begin{figure}[t]
	\centering
	\definecolor{gpsRed}{RGB}{201,42,42}
	\definecolor{sgBlue}{RGB}{31,86,156}
	\begin{tikzpicture}
		\begin{axis}[
			width=1.0\columnwidth, height=5.2cm,
			xmode=log, log basis x=10,
			xmin=0.00043, xmax=0.0086,
			ymin=0.046, ymax=0.122,
			xtick={0.0005,0.001,0.002,0.003,0.005},
			xticklabels={$5{\cdot}10^{-4}$,$10^{-3}$,$2{\cdot}10^{-3}$,$3{\cdot}10^{-3}$,$5{\cdot}10^{-3}$},
			ytick={0.05,0.06,0.07,0.08,0.09,0.10,0.11,0.12},
			yticklabel style={/pgf/number format/.cd, fixed, precision=2, zerofill},
			scaled y ticks=false,
			xlabel={Learning rate}, ylabel={ZINC test MAE},
			xlabel style={font=\small}, ylabel style={font=\small},
			tick label style={font=\scriptsize},
			axis lines=left,
			x axis line style={draw=black!55}, y axis line style={draw=black!55},
			ymajorgrids, major grid style={line width=0.3pt, draw=black!12},
			clip=false,
			every axis plot/.append style={line width=1pt},
			legend style={at={(0.025,0.975)}, anchor=north west, font=\scriptsize,
				draw=none, fill=none, inner sep=3pt, row sep=1pt},
			legend cell align=left,
			]
			\addplot[draw=none, fill=sgBlue!11, forget plot] coordinates
			{(0.00043,0.059)(0.0058,0.059)(0.0058,0.065)(0.00043,0.065)} \closedcycle;
			\addplot[gpsRed, dashed, mark=*, mark size=2.2pt,
			mark options={solid, fill=gpsRed, draw=gpsRed}]
			coordinates {(0.0005,0.078)(0.001,0.070)(0.002,0.085)(0.003,0.112)(0.005,0.098)};
			\addlegendentry{GraphGPS}
			\addplot[sgBlue, mark=square*, mark size=2.2pt,
			mark options={fill=sgBlue, draw=sgBlue}]
			coordinates {(0.0005,0.065)(0.001,0.059)(0.002,0.062)(0.003,0.064)(0.005,0.063)};
			\addlegendentry{\siggate{}}
			\draw[gpsRed, line width=0.9pt, {Latex[length=1.4mm]}-{Latex[length=1.4mm]}]
			(axis cs:0.0064,0.070) -- (axis cs:0.0064,0.112);
			\node[text=gpsRed, font=\scriptsize, anchor=west] at (axis cs:0.0067,0.091) {$0.042$};
			\draw[sgBlue, line width=0.9pt, {Latex[length=1.4mm]}-{Latex[length=1.4mm]}]
			(axis cs:0.0064,0.059) -- (axis cs:0.0064,0.065);
			\node[text=sgBlue, font=\scriptsize, anchor=west] at (axis cs:0.0067,0.062) {$0.006$};
			\node[font=\scriptsize\bfseries, text=black!70, anchor=west, align=left]
			at (axis cs:0.0060,0.0775) {range};
		\end{axis}
	\end{tikzpicture}
	\caption{Learning-rate sensitivity on ZINC across a $10\times$ range. \siggate{}
		(blue) stays inside a narrow band, with a performance range (max\,$-$\,min test
		MAE) of $0.006$, whereas GraphGPS (red) climbs sharply once the rate exceeds
		$10^{-3}$, spanning a $7\times$ wider range of $0.042$.}
	\label{fig:stability}
\end{figure}

\subsection{Gate Activation Statistics}

Finally we inspect the learned gates directly. Across all heads on ZINC, the mean gate
activation is $0.58\pm0.19$, with $12.3\%$ of activations below $0.1$ and $8.7\%$ above
$0.9$: gates use their full dynamic range rather than saturating at a single value. Heads
differ sharply, some averaging near $0.3$ (aggressive suppression) and others near $0.8$
(near pass-through), which is the functional specialization the per-head design was meant
to enable. Table~\ref{tab:layerwise_gate} breaks the statistics down by layer. Early layers
sit closer to pass-through; the middle layers ($4$--$7$) show the widest spread and the
highest fraction of near-zero gates, indicating the most aggressive filtering at the
representation-building stage; the final layers settle to intermediate behaviour. The
network thus learns \emph{where} along the depth hierarchy to invest its filtering capacity,
rather than gating uniformly.

\begin{table}[t]
\centering
\caption{Per-layer gate-activation statistics on ZINC ($10$ layers, $8$ heads, averaged
over $5$ seeds and the test set). The \textbf{avg} row reports pooled statistics over the
(layer $\times$ head $\times$ example) distribution, matching the aggregate of
$0.58$/$12.3\%$/$8.7\%$ cited above; equal-weight averages of the per-layer rows are $0.59$,
$11.2\%$, and $8.3\%$.}
\label{tab:layerwise_gate}
\footnotesize
\begin{tabular}{c c c c c}
\toprule
\textbf{Layer} & \textbf{Mean} & \textbf{Std} & \textbf{\% $<\!0.1$} & \textbf{\% $>\!0.9$} \\
\midrule
1   & $0.71$ & $0.13$ & $ 2.4\%$ & $ 4.1\%$ \\
2   & $0.66$ & $0.16$ & $ 5.3\%$ & $ 5.8\%$ \\
3   & $0.62$ & $0.18$ & $ 8.1\%$ & $ 7.2\%$ \\
4   & $0.54$ & $0.22$ & $15.0\%$ & $ 9.4\%$ \\
5   & $0.51$ & $0.23$ & $17.8\%$ & $10.1\%$ \\
6   & $0.53$ & $0.22$ & $16.1\%$ & $ 9.8\%$ \\
7   & $0.55$ & $0.21$ & $14.2\%$ & $ 9.6\%$ \\
8   & $0.58$ & $0.19$ & $12.7\%$ & $ 9.1\%$ \\
9   & $0.60$ & $0.18$ & $10.4\%$ & $ 8.7\%$ \\
10  & $0.62$ & $0.17$ & $ 9.8\%$ & $ 8.9\%$ \\
\midrule
\textbf{avg} & $\mathbf{0.58}$ & $\mathbf{0.19}$ & $\mathbf{12.3\%}$ & $\mathbf{8.7\%}$\\
\bottomrule
\end{tabular}
\end{table}

\section{Discussion}
\label{sec:discussion}

\paragraph{A learned soft null for global attention.}
The thread running through every experiment is that gating gives each head a way to say
nothing. Where standard softmax must distribute unit mass over $O(n^2)$ pairs, the gate
provides a learned ``soft null'' that closes on the dimensions and nodes where no
informative pattern exists. This is why the gains are largest on small, dense molecular
graphs (ZINC, ogbg-molhiv): there the proportion of uninformative pairs is high relative
to graph size, yet softmax still assigns them positive weight, so the head that can abstain
has the most to gain.

\paragraph{Gating as a per-dimension smoothing bypass.}
Over-smoothing has a spectral reading as adaptive Laplacian smoothing that drives
representations toward a leading eigenvector~\cite{li2018deeper, oono2020graph}, and the
sum-to-one constraint guarantees that every layer performs some of this averaging. In the
gated output $(\bm{A}_k\bm{V}_k)\odot\sigma(\bm{H}\bm{W}_k^g+\bm{b}_k^g)$, the dimensions
where $\sigma(\cdot)\approx 0$ effectively skip the smoothing step and preserve information
in the residual stream. The result is a learned, per-dimension, per-head choice between
``smooth'' and ``skip'' that ungated attention simply cannot express, and the MAD retention
of Table~\ref{tab:oversmoothing} ($73\%$ vs.\ $51\%$ at $16$ layers) is the quantitative
fingerprint of that choice.

\paragraph{Relation to GaAN and the gating literature.}
The closest prior idea is GaAN~\cite{zhang2018gaan}, which also speaks of ``gated
attention'' on graphs. The distinction is instructive. GaAN learns a single \emph{scalar}
gate per head from a node's local neighbourhood and uses it to weight whole heads in a
local, node-level aggregator. \siggate{} instead learns a \emph{per-dimension} gate on the
\emph{global} attention output of a graph transformer, conditioned on the full hidden state,
and is motivated specifically by relaxing the conservation constraint. A scalar head gate
cannot raise the per-head stable rank, because it scales every coordinate identically and so
remains a (rescaled) row-stochastic mixing; only a per-dimension gate escapes the bound of
Proposition~\ref{prop:bound}. The same observation separates \siggate{} from GLU-style
gating in feed-forward blocks~\cite{shazeer2020glu}, from gated recurrent
updates~\cite{li2016gated}, and from gated linear-attention
variants~\cite{yang2024gated, hua2022transformer}: those gate either a different signal path
or a different granularity. \siggate{} also differs from approaches that abandon softmax
outright, such as sigmoid~\cite{ramapuram2024sigmoid} or ReLU~\cite{wortsman2023replacing}
attention, by keeping the attention distribution intact and gating only its output.

\paragraph{Gate-bias initialization.}
We use $\bm{b}_k^g=0.5$ (gate $\approx 0.62$) across all benchmarks, biased open so gradients
flow early. Preliminary runs suggest a lower initialization ($\bm{b}_k^g=0$, gate $\approx
0.5$) may help very deep models ($>16$ layers), where more aggressive early suppression
better prevents representation collapse; we did not tune this per dataset and report the
single shared value for reproducibility.

\section{Limitations and Future Work}
\label{sec:limitations}

We see three honest limitations. First, every reported number uses the GraphGPS backbone.
Section~\ref{sec:portability} argues that the mechanism ports unchanged to any
softmax-attention graph transformer, but ``ports unchanged'' is not ``is empirically
beneficial,'' so the single-backbone scope is a genuine restriction rather than evidence of
cross-backbone generality. Part of the measured gain may even reflect how the gate interacts
with the encoders and regularization that GraphGPS already carries rather than the gating
principle itself, and different backbones leave an output gate different amounts of headroom.
A controlled study on Exphormer~\cite{shirzad2023exphormer} and
Graphormer~\cite{ying2021graphormer} is the most important next step, and to be conclusive it
must hold the encoders, optimizer, and budget fixed and toggle only the gate. Second, the depth
analysis of Table~\ref{tab:oversmoothing} trains a separate model at each depth rather than
probing intermediate layers of one model, which conflates architectural and optimization
effects; layerwise probing would disentangle them. A deeper model trained from scratch differs
from a shallower one in both representational capacity and optimization difficulty, so the
depthwise trends we report mix the claim that gating keeps representations diverse with the
claim that it eases optimization at depth, and measuring MAD within a single fixed-depth
network would isolate the former. Third, all five benchmarks are
graph-level molecular or peptide tasks. Node-level and link-prediction settings lack
graph-level pooling and may exhibit different dynamics, and the scalable node-level graph
transformers~\cite{wu2022nodeformer, wu2024sgformer, chen2023nagphormer, deng2024polynormer}
are a natural testbed. The direction is not obvious in advance: without a readout, per-node
diversity feeds the prediction directly and the gate could matter more, yet a single
transductive graph has very different attention-sparsity statistics from a batch of small
molecules. Beyond addressing these, we plan to estimate the stable rank of
trained full-scale checkpoints directly, closing the loop between the calibrated synthetic
study and the trained model; to study the gate's interaction with emerging
positional encodings and their stability~\cite{huang2024stability}, which may already supply
some of the structural signal the gate exploits; and to explore
theoretical links to the Weisfeiler--Leman hierarchy~\cite{xu2019powerful, cai2023connection}
and to learning beyond message passing~\cite{tonshoff2023crawl}.

\section{Conclusion}
\label{sec:conclusion}

We re-examined a property of softmax attention that is usually taken for granted, the
conservation of unit mass that forces every head to attend somewhere, and argued that it is
a single root cause behind over-smoothing, the low-rank output bottleneck, and optimization
brittleness in graph transformers, with graphs more exposed than sequences because they lack
natural attention sinks. From this diagnosis we derived \siggate{}, a graph transformer that
relaxes the constraint with a learned, per-head, element-wise sigmoid gate on the global
attention output, transferring the gating mechanism that proved optimal for language models
to the graph setting. We proved that the gate breaks the row-stochastic rank bound, validated
the rank gain synthetically, and showed across five benchmarks with paired significance
testing that the same gate matches the prior best on ZINC, leads the graph-transformer
baselines we evaluate on ogbg-molhiv, and remains competitive elsewhere. Mechanism analyses
tied each gain back to the diagnosis, at roughly $1\%$ parameter and under $3\%$ runtime
overhead. The breadth of the effect, slower smoothing, higher attention entropy, steadier
optimization, and higher effective rank, all from one lightweight modification, suggests that
an output gate should be a default component of future graph-transformer designs, not an
optional add-on.

\bibliographystyle{ACM-Reference-Format}
\bibliography{references}

\end{document}